\documentclass{IOS-Book-Article}

\usepackage{mathptmx}

\usepackage{graphicx}
\usepackage{enumitem}
\usepackage{subfig}
\usepackage{algpseudocode,algorithm,algorithmicx}
\usepackage{mathtools}
\usepackage[table]{xcolor}
\usepackage{multirow}
\usepackage{multicol}
\usepackage[flushleft]{threeparttable}
\usepackage{pifont}
\usepackage{array}
\usepackage{makecell}
\usepackage{longtable}
\usepackage[titletoc,title]{appendix}
\usepackage{lscape}
\usepackage{datatool}
\usepackage{forloop}
\usepackage{array}
\usepackage{commath}
\usepackage{hyperref}
\usepackage{environ}
\usepackage{url}

\def\hb{\hbox to 10.7 cm{}}

\newtheorem{definition}{Definition}
\newtheorem{proof}{Proof}
\newtheorem{theorem}{Theorem}

\newtheorem{lemma}[theorem]{Lemma}

\makeatletter
\@addtoreset{proofpart}{theorem}
\makeatother
\newcommand{\repeattheorem}[1]{%
	\begingroup
	\renewcommand{\thetheorem}{\ref{#1}}%
	\expandafter\expandafter\expandafter\theorem
	\csname reptheorem@#1\endcsname
	\endtheorem
	\endgroup
}

\NewEnviron{reptheorem}[1]{%
	\global\expandafter\xdef\csname reptheorem@#1\endcsname{%
		\unexpanded\expandafter{\BODY}%
	}%
	\expandafter\theorem\BODY\unskip\label{#1}\endtheorem
}

\setlist[itemize]{noitemsep, topsep=1mm}

\algnewcommand\algorithmicinput{\textbf{Input:}}
\algnewcommand\Input{\item[\algorithmicinput]}

\newcommand{\cmark}{\text{\ding{51}}}
\newcommand{\xmark}{\text{\ding{55}}}

\newcommand{\verytiny}{\fontsize{5}{6}\selectfont}
\newcommand{\ora}{\overrightarrow}
\newcommand{\WRP}{\par\qquad\(\hookrightarrow\)\enspace}

\newcounter{number}
\newcounter{rowcounter}
\newcommand*{\thevalue}{}
\newcommand*{\csvEight}[1]{%
	\setcounter{number}{0}
	\addtocounter{rowcounter}{1}
	\forloop[2]{number}{0}{\value{number} < 8}{%
		&\verytiny
		\ifthenelse{\isodd{\value{rowcounter}}}{\cellcolor{gray!25}}{\cellcolor{white}}
		\DTLgetvalue{\thevalue}{\dtllastloadeddb}{\the\numexpr #1 + \value{number}}{2}\thevalue \ifthenelse{\equal{\string\detokenize{\thevalue}}{\string\detokenize{N/A}}}{\hspace{-1.9mm}}{ms} 
		\newline \DTLgetvalue{\thevalue}{\dtllastloadeddb}{\the\numexpr #1 + \value{number} + 1}{3}\thevalue \ifthenelse{\equal{\string\detokenize{\thevalue}}{\string\detokenize{N/A}}}{}{MB}}
}

\newcolumntype{H}{>{\setbox0=\hbox\bgroup}c<{\egroup}@{}}

\begin{document}

\pagestyle{headings}
\def\thepage{}

\begin{frontmatter}              % The preamble begins here.

	%\pretitle{Pretitle}

	\title{Efficiently Checking Actual Causality \\with SAT Solving\thanks{Work supported by the German National Science Foundation DFG under grant no. PR1266/3-1, Design Paradigms for Societal-
Scale Cyber-Physical Systems.}}
	
	\markboth{}{Preprint\hb}
	%\subtitle{Subtitle}

%\thanks{Corresponding Author: Amjad Ibrahim,
%			Boltzmannstraße 3, 85748 Garching b. München, Germany; E-mail:
%			ibrahim@in.tum.de.}
	
	\author[A]{\fnms{Amjad} \snm{Ibrahim}%
		},
	\author[A]{\fnms{Simon} \snm{Rehwald}}
	and
	\author[A]{\fnms{Alexander} \snm{Pretschner}}
	
	\runningauthor{A. Ibrahim et al.}
	\address[A]{Department of Informatics, Technical University of Munich, Germany}
	
\begin{abstract}
	Recent formal approaches towards causality have made the concept ready for incorporation into the technical world. However, causality reasoning is computationally hard; and no general algorithmic approach exists that efficiently infers the causes for effects. Thus, checking causality in the context of complex, multi-agent, and distributed socio-technical systems is a significant challenge. Therefore, we conceptualize an intelligent and novel algorithmic approach towards checking causality in acyclic causal models with binary variables, utilizing the optimization power in the solvers of the Boolean Satisfiability Problem (SAT). We present two SAT encodings, and an empirical evaluation of their efficiency and scalability. We show that  causality is computed efficiently in less than 5 seconds for  models that consist of more than 4000 variables. %We also extend the algorithms to either provide a better quality of the answer or optimize the efficiency.   

\end{abstract}
\end{frontmatter}
\markboth{Preprint\hb}{Preprint\hb}
\section{Introduction}\label{sec:introduction}
\subsection{Accountability}
The tremendous power of modern information technology, and cyber-physical, systems is rooted, among other things, in the possibility to easily compose them using software. The ease of composition is due to the virtual nature of software: if a system provides an application programming interface, then other systems can in principle directly interact with that system. As we will argue, it is inherently impractical to specify all ``legal,'' or ``safe,'' or ``secure'' interactions with a system. In turn, this means that the possibility of illegal, unsafe or insecure interactions cannot be excluded at design time. As a consequence, we cannot ensure the ``adequate'' functioning of such open systems; hence need to be prepared for failures of the system; and therefore need \emph{accountability mechanisms} that help us identify the root cause, or responsibility, for such a failure. The technical contribution of this article is an efficient procedure for checking a specific kind of causality.  

Traditionally, safety and also security analysis of critical systems start by determining the system's boundary. We will concentrate on the technical boundary here. Any accessible software interface, or API, then is part of that boundary. In practice, these interface descriptions are syntactic in nature.
%\footnote{Arguably, the types of parameters also provide an abstraction of the ``semantics'' of a method or a service. We will nonetheless stick to the common wording that calls ``syntactic'' an interface description that merely maps types of input parameters to types of output parameters.} 
Among other things, this means that they abstract from restrictions on the usage of the API. For instance, a collection of data may need to be initialized before it is used, which is a requirement that is not part of the syntactic interface but may require additional specification in the form of sequencing requirements.

This example generalizes: Many restrictions on the potential usage of an API are left implicit, or come as comments in the best case. The academic community has therefore suggested, for a long time, to provide more detailed interface descriptions. A typical example for such interfaces are \emph{contracts} \cite{DBLP:journals/computer/Meyer92} that require the specification of preconditions and postconditions for the usage of a specific piece of functionality.

In practice, today's software interfaces continue to be mostly syntactic, in spite of decades of research and impressive results on richer notions of interfaces that also incorporate (more detailed descriptions of) the behavior of a component. We prefer not to speculate about the lack of adoption of these ideas in practice. Instead, we would like to remind that \emph{any} interface description provides a behavior abstraction---and in this sense, the syntactic interface, or API, provides such a coarse abstraction as well: data elements of a certain type are mapped to data elements of another type. 

Arguably, this is the coarsest abstraction of behavior that still is useful in practice. At the other end of the spectrum of levels of abstraction, one may argue that the finest possible abstraction is that of the code itself. However, we do share the perspective that code itself provides an abstraction of behavior in that in most programming languages it does not explicitly talk about time or resource utilization. In this sense, code is not the finest possible abstraction of behavior. Be that as it may, the above shows that there is a huge spectrum of possible levels of abstraction for describing the behavior, or an interface, of a system. It is important to remember that none of these levels as such is ``better'' than another level: this depends on the purpose of the abstraction, as is the case for any kind of model \cite{stachowiak}.

All this does not necessarily constitute novel insights. There is a consequence, however, that we feel has been underestimated in the past and that is the basis for the work presented here: Regardless of the level of abstraction of an interface that we choose, there must, by definition, always be parts of the behavior that are left unspecified. And this, in turn, means that the boundary of a software-intensive system usually is not, should not, and most importantly: \emph{cannot} be specified without leaving open several degrees of freedom in terms of how to ``legally'' use the system. It hence cannot be excluded that a system $S_1$ is used by another system $S_2$ in a way that the developer of $S_1$ has never envisaged and which violates implicit assumptions that were made when developing $S_1$, possibly leading to a runtime failure of $S_1$ and, by consequence, also of $S_2$.\footnote{We believe that this construction helps us understand what the notion of an ``open system'' means: At face value, a software-intensive system $S$ cannot be ``open'' because there is a well-specified (syntactic) interface for the behavior of $S$. However, if this interface description is too coarse and cannot or does not specify restrictions on how to use $S$, some usage of $S$ may violate implicit assumptions made while developing $S$. In this sense, \emph{every} software-intensive system is then ``open'' if it may be used in a way that violates implicit assumptions.}  

One consequence is that software-rooted failures for composed, or open, systems cannot be excluded by design. Because these systems are becoming ever more complex, we consider it hence mandatory to provide mechanisms for understanding the root causes of a failure, both for eliminating the underlying (technical) problem and for assigning blame.

We call a system \emph{accountable} that can help answer questions regarding the root cause of some (usually undesired) event (other notions of accountability are studied and formalized elsewhere \cite{kacianka2018understanding}). Accountability of a system requires at least two properties: The system must provide \emph{evidence}, usually in the form of logs, that helps understand the reasons of the undesired event. Moreover, there must be some mechanism that can reason about \emph{causality}.

Different forms of causality are useful in our context. Just to mention two of them, Granger causality identifies causality with a reduction of error when it comes to predicting one time series on the grounds of another \cite{granger:80}; and model-based diagnosis computes candidate sets of potentially faulty components that can explain the failure of a system \cite{DBLP:journals/ai/Reiter87}. In this article, we focus on one technical aspect of inferring one specific kind of causality, namely Halpern-Pearl-style actual causality for counter-factual arguments: Given a failure of a system (an effect) and a potential cause, we efficiently compute if, by counterfactual argument, the potential cause is an actual cause. 

\subsection{Causality}
Causality is a fundamental construct of  human perception. Although our ability to link effects to causes may seem natural, defining what precisely constitutes a cause has baffled scholars for centuries. Early work on defining causality goes back to Hume in the eighteenth century \cite{hume1748An}. Hume's definition hinted at the idea of counter-factual relations to infer the causes of effects. Informally, we argue, counter to the fact, that $A$ is a cause of $B$ if $B$ does not occur if $A$ no longer occurs. As Lewis noted with examples, this relation is not sufficient to deal with interdependent, multi-factorial, and complex causes \cite{lewis1973causation}. Thus, the search for a comprehensive general definition of causality continues. Recently, in computer science, there have been some successful and influential efforts, by Halpern and Pearl, at formalizing the idea of counter-factual reasoning and addressing the problematic examples in  philosophy literature \cite{halpern2016actual}.

The work by Halpern and Pearl covers two notions of causality, namely actual (token) causality and type (general) causality. Type causality is a forward-looking link that forecasts the effects of causes. It is useful for predictions in different domains like medicine \cite{kleinberg2011review}, and machine learning applications \cite{pearl2018theoretical}. In this paper, we focus on actual causality that is a rather retrospective linking of effects to causes. We are chiefly interested in the Halpern-Pearl (HP) definition of actual causality \cite{halpern2015modification}. Causality is useful in law \cite{moore2009causation},  security \cite{kuennemann2018automated}, software and system verification \cite{beer2012explaining,leitner-fischer2013causality}, databases \cite{meliou2010causality}, and accountability \cite{feigenbaum2011towards,kacianka2016towards}.

In essence, HP provides a formal definition of when we can call one or more events a cause of another event in a way that captures the human intuition. There have been three versions of HP: the original (2001), updated (2005), and modified (2015) versions, the latter of which we are using. The fundamental contribution of HP is that it opens the door for embedding the ability to reason about causality into socio-technical systems that are increasingly taking control of our daily lives. Among other use cases, since actual causality can be used to answer causal queries in the postmortem of unwanted behavior, it is a vital ingredient to enable accountability. Utilizing HP in technical systems makes it possible to empower them with all the other social concepts that should also be embedded into the technical world, such as responsibility \cite{chockler2004responsibility}, blame, intention, and moral responsibility \cite{halpern2018towards}.

Causality checking, using any version of  HP, is hard. For example, under the \textit{modified} definition, determining causality is in general $D^P_1$-complete, and $NP$-complete for singleton causes \cite{halpern2015modification}. The computational complexity led to a domain-specific (e.g., database queries,  counter-examples of model checking), adapted (e.g., use lineage of queries, use Kripke structure of programs), or restricted (e.g., monotone queries, singleton causes, single-equation models) utilization of HP (details in Section.\ref{sec:realted}). Conversely, brute-force approaches work with small  models (less than 30 variables \cite{hopkins2002strategies}) only. Therefore, to the best of our knowledge, there exists no comprehensive, efficient, and scalable  framework for modeling and benchmarking causality checking for binary models (i.e., models with binary variables only). Consequently, no existing algorithm allows applying HP on more complex examples than the simple cases in the literature.

In this paper, we argue that an efficient approach for checking causality opens the door for new use cases that leverage the concept in modern socio-technical systems. We conceptualize a novel approach towards checking causality in acyclic binary models based on the Boolean satisfiability problem (SAT). We intelligently encode  the core of HP as a SAT query that allows us to reuse the optimization power built into SAT solvers. As a consequence of the rapid development of SAT solvers ($1000X+$ each decade), they offer a promising tactic for any solution in formal methods and program analysis \cite{newsham2014impact}. Leveraging this power in causality establishes a robust framework for efficiently  reasoning  about actual causality. Moreover, since the transformation of SAT to other logic programming paradigms like answer set programming (ASP) is almost straightforward, this paper establishes the ground to tackle more causality issues (e.g., causality inference) using combinatorial solving approaches. Therefore, this paper makes the following contributions:
\begin{itemize}
	\item An approach to check causality over binary models. It includes two SAT-encodings that reflect HP, and two variants for optimization,
	\item A Java library \footnote {available at: \url{https://github.com/amjadKhalifah/HP2SAT1.0/}} that includes the implementation of our approach. It is easily extensible with new optimizations and algorithms.
	\item An empirical evaluation that uses different examples to show the efficiency and scalability of our approach.
\end{itemize}

\section{Related Work} \label{sec:realted}
To the best of our knowledge, no previous work has tackled the technical implementation of the (\textit{modified}) version of HP yet. Conversely, the first \textit{two} versions were used in different applications. Although they use different versions, we still consider them related. These applications used the definition as a refinement of other technologies. For example, in \cite{meliou2010causality,meliou2010complexity,bertossi2018characterizing,salimi2014causes}, a simplified HP (the updated version) was used to refine provenance information and explain database conjunctive query results.  Theses approaches, heavily depend on the correspondence between causes and domain-specific concepts like lineage, database repairs, and denial constraints. As a simplification, the authors used a single-equation causal model based on the lineage of the query in \cite{meliou2010causality,meliou2010complexity}, and no-equation model in \cite{bertossi2018characterizing,salimi2014causes}. The approaches also eliminate HP's treatment of preemption. 
Similar simplification has been made for Boolean circuits in \cite{chockler2008causes}.

In the context of model verification,  a concept that enhances a counterexample returned by a model checker with a causal explanation based on a simplified version of the updated HP was proposed in \cite{beer2012explaining}. Their domain-specific simplification comes from the fact that no dependencies between variables, and hence no equations, were required. Moreover, they used the definition of singleton causes. %Their approach is integrated into a tool called \textit{RuleBase PE}, a formal verification platform by IBM.
Similarly, in \cite{beer2015symbolic,leitner-fischer2015causality}, the authors implemented different flavors of causality checking (based on the updated HP) using Bounded Model Checking to debug models of safety-critical systems. They employed SAT solving indirectly in the course of model checking. The authors stated that this approach is better in large models regarding performance than their previous work.

The common ground between all these approaches and our approach is the usage of binary models. However, they were published before the  modified HP version. Hence, they used the older versions. In contrast to our approach, the previous works adapted the definition for a domain specific purpose. This was reflected in restrictions on the equations of the binary model (single-equation, independent variables), the cause (singleton), or dependency on other concepts (Kripke structures, lineage formula, counter-examples). On the contrary, we propose algorithms to compute causality on binary models, without adaptations.

Similar to our aim, \cite{hopkins2002strategies} evaluated search-based strategies for determining  causality according to the original HP definition. Hopkins proposed ways to explore and prune the search space, for computing $\vec{W},\vec{Z}$ that were required for that version, 
and considered properties of the causal model that makes it more efficient for computation. The results presented are of models that consist of less than 30 variables; in contrast, we show SAT-based strategies that compute causality for models of thousands of variables.

Lastly, to prove the complexity classes, Halpern \cite{halpern2015modification} used the relation between the conditions and the SAT problem. However, the concrete encoding in SAT with a size that is linear of the number of variables is still missing.

  %Despite the above reasons, we agree with reviewer#1 that causality can be reduced to other (than SAT) solving problems, however, so far these reductions utilize domain-specific abilities that are not available in other use-cases. This domain specificity hinders taking advantage of the available approximation algorithms for general causal queries. 

\section{Halpern-Pearl Definition}\label{sec:preliminaries}
In this section, we introduce the latest HP. All versions of HP use variables to describe the world. Structural equations define how these variables influence each other. The variables are split into \textit{exogenous} and \textit{endogenous} variables. The values of the former, called a \textit{context $\vec{u}$}, are assumed to be governed by factors that are not part of the modeled world. Consequently, exogenous variables cannot be part of a cause. The values of the endogenous variables, in contrast, are determined by the mentioned equations.  Formally, we describe a causal model in Definition \ref{def:causal_model}. Similar to Halpern, we limit ourselves to acyclic causal models, in which we can compute a unique solution for the equations given a context $\vec{u}$.
\begin{definition}\label{def:causal_model}
	\textbf{Causal Model}
	is a tuple $M = (U,V,R,F)$, where
	\begin{itemize}
		\item $U$, ${V}$ are sets of exogenous variables and endogenous variables respectively, 
		\item $R$ associates with $Y \in U \cup V$ a set $R(Y)$ of values,
		\item $F$ associates with $X \in V$  $F_X : (\times_{U \in U}{R}(U)) \times (\times_{Y \in {V}\backslash\{X\}}{R}(Y))\to {R}(X)$			
	\end{itemize}
\end{definition}
 Here, we define the necessary notations. 	%A \textit{primitive event} is a formula of the form $X = x$, for $X \in \mathcal{V}$ and $x \in \mathcal{R}(X)$. A sequence of variables $X_1,...,X_n$ is abbreviated as $\vec{X}$. Analogously, $X_1=x_1,...,X_n=x_n$ is abbreviated $\vec{X}=\vec{x}$.  $\varphi$ is a boolean combination of primitive events. Variable Y can be set to value y  writing $Y \leftarrow y$ (analogously $\vec{Y} \leftarrow \vec{y}$ for vectors). Then, a causal formula is one of the form $[Y_1 \leftarrow y_1, ..., Y_k \leftarrow y_k]\varphi$, where $Y_1, ..., Y_k$ are distinct variables in $V$. $(M, \vec{u}) \models X = x$ if the variable $X$ has value $x$ in the unique solution to the equations in $M$ in context $\vec{u}$. We write $(M, \vec{u}) \models \phi$ if the causal formula $\phi$ is true in $M$ given context $\vec{u}$. Lastly, $(M, \vec{u}) \models [\vec{Y} \leftarrow \vec{y}]\varphi$ implies $(M_{\vec{Y} \leftarrow \vec{y}},\vec{u}) \models \varphi$.
	 A \textit{primitive event} is a formula of the form $X = x$, for $X \in {V}$ and $x \in {R}(X)$. A sequence of variables $X_1,...,X_n$ is abbreviated as $\vec{X}$. Analogously, $X_1=x_1,...,X_n=x_n$ is abbreviated as $\vec{X}=\vec{x}$.  Variable Y can be set to value y by writing $Y \leftarrow y$ (analogously $\vec{Y} \leftarrow \vec{y}$ for vectors). $\varphi$ is a Boolean combination of primitive events. $(M, \vec{u}) \models X = x$ if the variable $X$ has value $x$ in the unique solution to the equations in $M$ in context $\vec{u}$. Intervention on a model is expressed, either by setting the values of $\vec{X}$ to $\vec{x}$, written as $[X_1 \leftarrow x_1, .., X_k \leftarrow x_k]$, or by fixing the values of $\vec X$ in the model, written as $M_{\vec{X} \leftarrow \vec{x}}$, which effectively replaces the equations for $\vec{X}$ by a constant equation $X_i=x_i$. So, $(M, \vec{u}) \models [\vec{Y} \leftarrow \vec{y}]\varphi$ is identical to $(M_{\vec{Y} \leftarrow \vec{y}},\vec{u}) \models \varphi$ \cite{halpern2015modification}. Lastly, we use $\mapsto$ to express value substitution, e.g., $[\overrightarrow{V} \mapsto \vec{v}'] F_{X_j}$ refers to the evaluation of equation $F_{X_j}$ given that the values of other variables are set to $\vec{v}'$.
	%\item A causal formula is of the form $[Y_1 \leftarrow y_1, ..., Y_k \leftarrow y_k]\varphi$.

\begin{definition} \label{def:ac}
	\textbf{Actual Cause} \cite{halpern2015modification}\\ $\vec{X} = \vec{x}$
	is an actual cause of $\varphi$ in $(M,\vec{u})$ if the following
	three conditions hold:\\ \textbf{AC1.} $(M, \vec{u}) \models (\vec{X} = \vec{x})$ and $(M, \vec{u}) \models \varphi$\\ \textbf{AC2.} There is a set $\vec{W}$ of variables in ${V}$ and a setting $\vec{x}'$ of the variables in $\vec{X}$ such that if $(M, \vec{u}) \models \vec{W} = \vec{w}$, then 	$(M, \vec{u}) \models [\vec{X} \leftarrow \vec{x}', \vec{W} \leftarrow \vec{w}]\neg\varphi$.\\ \textbf{AC3.} $\overrightarrow{X}$ is minimal: no subset of $\vec{X}$ satisfies AC1 and AC2.
\end{definition}
 The HP definition is presented in Definition \ref{def:ac}.   AC1 checks that the cause and the effect occurred in the real world, i.e., in $M$ given context $\vec{u}$. AC3 is a minimality check to deal with irrelevant variables. AC2 is the core; it matches the counter-factual definition of causality. It holds if there exists a \textit{setting} $\vec{x}'$ of the variables in $\vec{X}$ different from the \textit{original} setting $\vec{x}$ (which led to $\varphi$ holding true) and another set of variables $\vec{W}$ that we use to \textit{fix} variables at their original value, such that $\varphi$ does not occur. Inferring $\vec{W}$ is one source of the complexity of the definition. The role of $\vec{W}$ becomes clearer when we consider the examples in \cite{halpern2015modification}. Briefly, it captures the notion of \textit{preemption} which describes the case when  one possible cause rules out the other based on, e.g., temporal factors.
%  Lemma.\ref{lemma:negation} clarifies what  precisely is the \textit{setting} $\vec{x}'$ for binary models. The proofs of this paper are in  \cite{lemmaProof}.
%  %occurs earlier and thereby preempts the other cause. 
%
%
%\begin{lemma}\label{lemma:negation}
%	In a binary model, if $\vec{X}=\vec{x}$ is a cause of $\varphi$, according to Def.\ref{def:ac}, then every $\vec{x}'$ in the definition of AC2 always satisfies $\forall i.x_i'=\neg x_i$.
%\end{lemma}

Halpern \cite{halpern2015modification} shows that determining causality is in general $D^P_1$-complete. The family of complexity classes $D^P_k$  was introduced, in \cite{aleksandrowicz2014computational}, to investigate the complexity of the \textit{original} and \textit{updated} definitions. \textit{AC1} can be checked in polynomial time, while \textit{AC2} is $NP$-complete, and \textit{AC3} is co-$NP$-complete. To prove this complexity, Halpern \cite{halpern2015modification} showed that AC2 could be reduced to SAT, and AC3 to UNSAT. However, the concrete encoding was not specified.

\textbf{Example} We consider a famous example from the literature: \textit{the  rock-throwing example}  \cite{lewis1973causation}, described as follows: Suzy and Billy both throw a rock at a bottle which shatters if one of them hits it. We know that Suzy's rock hits the bottle slightly earlier than Billy's and both are accurate throwers. Halpern models this story using the following endogenous variables:
$ST,$ $BT$ for ``Suzy/Billy throws'', with values 0 (the person does not throw) and 1 (s/he does), similarly, $SH, BH$ for ``Suzy/Billy hits'', and $BS$ for ``bottle shatters''. The equations are:\\
\begin{small}
	\indent	- $BS$ is 1 iff one of $SH$ and $BH$ is 1, i.e., $BS = SH \lor BH$ \\
	\indent	- $SH$ is 1 iff $ST$ is 1, i.e., $SH = ST$ \\
	\indent	- $BH = 1$ iff $BT = 1$ and $SH = 0$, i.e., $BH = BT \land \neg SH$ \\
	\indent	- $ST$,  $BT$  are set by exogenous variables, i.e., $ST = ST_{exo};BT = BT_{exo}$
\end{small}

Assuming a context $\vec{u}$ that sets $ST$ = 1 and $BT$ = 1, the original evaluation of the model is: 	$BS$=1	$SH$=1	$BH$=0	$ST$=1 $BT$=1. Let us assume we want to find out whether $ST$ = 1 is a cause of $BS$ = 1.  Obviously, AC1 is fulfilled. As a candidate cause, we set  $ST$ = 0. A first attempt with $\vec{W} = \emptyset$ shows that AC2 does \textit{not} hold. However, if we arbitrary take $\vec{W} = \{BH\}$, i.e., we replace the equation of $BH$ with $BH$ = 0, then AC2 holds because $BS$ = 0, and AC3 automatically holds since the cause is a singleton. Thus, $ST=1$ is a cause of $BS=1$.

\section{Approach}\label{sec:approach}
In this section, we propose our algorithmic approaches towards the HP definition. To answer a causal question efficiently, we need to find an intelligent way to search for a $\vec{W}$ such that AC2 is fulfilled as well as to check whether AC3 holds. Therefore, we propose an approach that uses SAT-solving. We show how to encode AC2 into a formula whose (un)satisfiability and thus the (un)fulfillment of AC2 is determined by a SAT-solver. Similarly, we show how to generate a formula whose satisfying assignments obtained with a solver indicate if AC3 holds. 

\subsection{Checking AC2}\label{subsec:ac2}
For AC2, such a formula $F$ has to incorporate (1) $\neg\varphi$, (2) a context $\vec{u}$, (3) a setting, $\vec{x}'$ for candidate cause, $\vec{X}$, and (4) all possible variations of $\vec{W}$, while still (5) keeping the semantics of the underlying model $M$. In the following, we describe the concept and, then, the algorithm that generates such a formula $F$. Since we check actual causality in hindsight, we have a situation where $\vec{u}$  and $\vec{v}$ are determined, and we have a candidate cause $\vec{X} \subseteq \vec{V}$. Thus, the first two requirements are straightforward. First, the effect $\varphi$ should not hold anymore, hence, $\neg\varphi$ holds. Second, the context $\vec{u}$ should be set to its values in the original assignment (the values $\vec{u}$  of $\vec{U}$). 

Since we are treating binary models only, the setting $\vec{x}'$ (from AC2) can be tailored down to negating the original value of each cause variable. This is a result of Lemma \ref{lemma:negation}, which utilizes the fact that we are considering binary variables to exclude other possible settings and define precisely the \textit{setting} $\vec{x}'$. The proof of the Lemma is given in Appendix \ref{sec:app1}. Thus, to address the third requirement, according to Lemma \ref{lemma:negation}, for $\neg\varphi$ to hold, all the variables of the candidate cause $\vec{X} $ are negated. 

\begin{lemma}\label{lemma:negation}
	In a binary model, if $\vec{X}=\vec{x}$ is a cause of $\varphi$, according to Definition \ref{def:ac}, then every $\vec{x}'$ in the definition of AC2 always satisfies $\forall i.x_i'=\neg x_i$.
\end{lemma}

To ensure that the semantics of the model are reflected in $F$ (requirement 5), we use the logical equivalence operator ($\leftrightarrow$) to express the equations. Particularly, to represent the endogenous variable $V_i$ and its dependency on other variables, we use this clause $V_i \leftrightarrow F_{V_i}$. This way, we create a (sub-)formula that evaluates to true if both sides are \textit{equivalent} in their evaluation. If we do so for all other variables (that are not affected by criteria 1-3), we ensure that $F$ is only satisfiable for assignments that respect the semantics of the model. 

Finally, we need to find a possibility to account for $\vec{W}$ (requirement 4) without having to iterate over the power-set of all variables. In $F$, we accomplish this by adding a disjunction with the positive or negative literal of each variable $V_i$ to the previously described equivalence-formula, depending on whether the actual evaluation of $V_i$ was $1$ or $0$, respectively. Then, we can interpret $((V_i \leftrightarrow F_{V_i}) \lor (\neg) V_i)$ as ``$V_i$ either follows the semantics of $M$ or takes on its original value represented as a positive or negative literal''. By doing so for all endogenous variables, we allow for all possible variations of $\vec{W}$. It is worth noting that we exclude those variables that are in $\vec{X}$ from obtaining their original value, as we are already changed their setting to $\neg\vec{x}$ and thus keeping a potential cause at its original value is not reasonable. Obviously, it might not make sense to always add the original value for all variables. We leave this as a candidate optimization for a future work.
\paragraph{\textbf{AC2 Algorithm}}
\begin{algorithm} [t]
	\caption{Check whether AC2 holds using SAT}\label{algorithm:fulfillsac2_sat}
	\begin{algorithmic}[1]
		\Input causal model $M$, context $\langle U_1,\ldots,U_n\rangle=\langle u_1,\ldots,u_n\rangle$, effect $\varphi$, candidate cause $\langle X_1,\ldots,X_\ell\rangle = \langle x_1,\ldots,x_\ell\rangle$, evaluation $\langle V_1,\ldots,V_m\rangle=\langle v_1,\ldots,v_m\rangle$
		\Function{FulfillsAC2}{$M, \vec{U}=\vec{u}, \varphi, \vec{X} = \vec{x},\vec{V}=\vec{v}$}
		\If{$(M, \vec{u}) \models [\vec{X} \leftarrow \neg\vec{x}]\neg\varphi$}\label{alg:line:fulfillsac2_sat:empty_w}
		\Return{$\emptyset$}
		\Else 	
		\State{ \label{alg:line:fulfillsac2_sat:f}
			$F := \neg\varphi \wedge \bigwedge\limits_{i=1\ldots n} f(U_i=u_i)$ 	$\wedge\bigwedge\limits_{i=1\ldots m, \not\exists j\bullet X_j=V_i} \left(V_i \leftrightarrow F_{V_i} \lor f(V_i=v_i)\right)$ \WRP	$ \wedge\bigwedge\limits_{i=1\ldots\ell} f(X_i=\neg x_i)$
			\\with $f({Y}={y}) = \begin{cases}
			Y,			&y = 1\\
			\neg Y,	&y=0
			\end{cases}$}
		\If{$\langle U_1=u_1\ldots U_n=u_n,V_1=v_1'\ldots V_m=v_m'\rangle \in$  $\textit{SAT}(\textit{CNF}(F))$}\label{alg:line:fulfillsac2_sat:assignment}
		\State $\vec{W} := \langle W_1,\ldots,W_s\rangle$ s.t. $\forall i\forall j\bullet (i\not= j\Rightarrow$ $W_i\not= W_j) \wedge (W_i=V_j\Leftrightarrow v_j'=v_j)$\label{alg:line:fulfillsac2_sat:w}
		\State \Return{$\vec{W}$}
		\Else{}
		\Return{\textit{not satisfiable}}
		\EndIf
		\EndIf
		\EndFunction
	\end{algorithmic}
\end{algorithm}

We formalize the above in Algorithm \ref{algorithm:fulfillsac2_sat}. The evaluation, in the input, is a list of all the variables in $M$ and their values under $\vec{u}$. The rest is self-explanatory. We slightly change the definition of $\varphi$ from a combination of primitive events to a combination of literals. For instance, instead of writing $\varphi = (X_1 = 1 \land X_2 = 0 \lor X_3 = 1)$, we would use $\varphi = (X_1 \land \neg X_2 \lor X_3)$. 
In other words, we replace each primitive $(X = x) \in \varphi$ with $X$ if $x = 1$ or $\neg X$ if $x = 0$ in the original assignment, such that we use $\varphi$ in a formula. The same logic is achieved using the function $f({Y}={y})$ in line 5 of the algorithm.

Before we construct formula $F$, we check if $\vec{X} = \vec{x}$ given $\vec{W} = \emptyset$ (line \ref{alg:line:fulfillsac2_sat:empty_w}) fulfills AC2. Hence, in this case, we do not need to look for a $\vec{W}$. Otherwise, we construct $F$ (line \ref{alg:line:fulfillsac2_sat:f}) that is a conjunction of $\neg\varphi$ and the exogenous variables of $M$ as literals depending on $\vec{u}$. Note that $\varphi$ does not necessarily consist of a single variable only; it can be any Boolean formula. For example, if $\varphi = (BS = 1 \land BH = 0)$ in the notation as defined by \cite{halpern2015modification}, we would represent it in $F$ as $(BS \land \neg BH)$.
This consideration is handled by Algorithm \ref{algorithm:fulfillsac2_sat} without further modification. In addition, we represent each endogenous variable, $V_i \not\in \vec{X}$ with a disjunction between its equivalence formula $V_i \leftrightarrow F_{V_i}$ and its literal representation. To conclude the formula construction, we add the negation of the candidate cause $\vec{X}=\vec{x}$, a consequence of Lemma \ref{lemma:negation}. 

If $F$, represented in a conjunctive normal form, is satisfiable, we obtain the satisfying assignment (line \ref{alg:line:fulfillsac2_sat:assignment}) and compute $\vec{W}$ (line \ref{alg:line:fulfillsac2_sat:w}) as the set of those variables whose valuations were \textit{not} changed in order to ensure $\neg\varphi$ that is finally returned. If $F$ is unsatisfiable, \textit{not satisfiable} is returned. 

\paragraph{\textbf{Minimality of $\vec{W}$}} In Algorithm \ref{algorithm:fulfillsac2_sat}, we considered $\vec{W}$ to consist of all the variables whose original evaluation and satisfying assignments are equal. This is an  over-approximation of the $\vec{W}$ set because, possibly, there are variables that are not affected by changing the values of the cause, and are yet not required to be fixed in $\vec{W}$. Despite this consideration, a non-minimal $\vec{W}$ is still valid according to HP. However, to compute the degree of responsibility \cite{chockler2004responsibility}, a minimal $\vec{W}$ is required. Therefore, we briefly discuss a modification that yields a minimal $\vec{W}$.

We need to modify two parts of Algorithm \ref{algorithm:fulfillsac2_sat}. First, we cannot just consider \textit{one} satisfying assignment for $F$. Rather, we need to analyze \textit{all} the assignments. Determining all the assignments is called an All-SAT problem. Second, we have to further analyze each assignment of $\vec{W}$ to check if we can find a subset such that $F$, and thus AC2, still holds. Specifically, we check if each element in $\vec{W}$ is equal to its original value because it was explicitly set so, or because it simply evaluated according to its equation. In the latter case, it is \textit{not} a required part of $\vec{W}$. Precisely, in Algorithm \ref{algorithm:fulfillsac2_sat}, everything stays the same until the computation of $F$. After that, we check whether $F$ is satisfiable, but now we compute all the satisfying assignments. Subsequently, for each satisfying assignment, we compute $\vec{W_i}$ as explained. Then, we return the smallest $\vec{W_i}$, at the cost of iterating over \textit{all} satisfying assignments of the variables in $V$.

%\begin{reptheorem}{ac2}\label{theorem:F}
%	Formula $F$ constructed within Algorithm.\ref{algorithm:fulfillsac2_sat} is satisfiable iff AC2 holds for a given $M$, $\vec{u}$, a candidate cause $\vec{X}$, and a combination of events $\varphi$. 
%\end{reptheorem}
%
%
%\input{sections/proof.tex}
\subsection{Checking AC3}\label{subsec:ac3}
Our approach for checking AC3 using SAT is very similar to the one for AC2. We construct another SAT formula, $G$. The difference between $G$ and $F$ is in how the parts of the cause are represented. In $G$, we allow each of them to take on its original value \textit{or} its negation (e.g., $A \lor \neg A$). Clearly, we could replace that disjunction with \textit{true} or $1$. However, we explicitly do not perform this simplification such that a satisfying assignment for $G$ still contains all variables of the causal model, $M$. 

In general, the idea is as follows. If we find a satisfying assignment for $G$ such that at least one conjunct of the cause $\vec{X} = \vec{x}$ takes on a value that equals  the one computed from its corresponding equation, then, we know that this particular conjunct is not required to be part of the cause and there exists a subset of $\vec{X}$ that fulfills AC2 as well. The same applies if the conjunct is equal to its original value in the satisfying assignment; this would mean that it is part of a $\vec{W}$ such that $\neg\varphi$ holds. When collecting all those conjuncts, we can construct a new cause $\vec{X}' = \vec{x}'$ by subtracting them from the original cause and then checking whether or not it fulfills AC1. If it does, AC3 is violated because we identified a subset $\vec{X}'$ of $\vec{X}$ for which both AC1 and AC2 hold. 

\paragraph{\textbf{AC3 Algorithm}}

We formalize our approach in Algorithm \ref{algorithm:fulfillsac3_sat}. The input and the function $f(V_i=v_i)$ remain the same as for Algorithm \ref{algorithm:fulfillsac2_sat}; the latter is omitted. In case $\vec{X} = \vec{x}$ is a singleton cause or $\varphi$ did not occur, AC3 is then fulfilled automatically (line \ref{alg:line:fulfillsac3_sat:singleton_phi}). Otherwise, line \ref{alg:line:fulfillsac3_sat:g} shows how formula $G$ is constructed. This construction is only different from the construction of $F$ in Algorithm \ref{algorithm:fulfillsac2_sat} in how to treat variables $\in\vec{X}$. They are added as a disjunction of their positive and negative literals. Once $G$ is constructed, we check its satisfiability, if it is not satisfiable we return \textit{true}, i.e., AC3 is fulfilled. For example, this can be the case if the candidate cause $\vec{X}$ did not satisfy AC2. Otherwise, we check \textit{all} its satisfying assignments. We need to do this, as $G$ might also be satisfiable for the original $\vec{X} = \vec{x}$ so that we cannot say for sure that any satisfying assignment found, proves that there exists a subset of the cause. Instead, we need to obtain all of them. Obviously, this is problematic and could decrease the performance if $G$ is satisfiable for a large number of assignments. Therefore, we plan to address this in future work.

 However, for now, we compute one assignment and check the \textit{count} of the conjuncts in the cause that have different values in $\vec{v'}$ than their original, and that their formula does not evaluate to this assignment (line \ref{alg:line:fulfillsac3_sat:new_cause}).   If the \textit{count} is less than the size of the cause, then AC3 is violated. Otherwise we check another assignment. %Theorem.\ref{theo:ac3} proofs the soundness of Algorithm.\ref{algorithm:fulfillsac3_sat}. %( proof is in Appendix.\ref{sec:ac3_proof}).  

\begin{algorithm}[t]
	\caption{Check whether AC3 holds using ALL-SAT}\label{algorithm:fulfillsac3_sat}
	\begin{algorithmic}[1]
		\Input{causal model $M$, context $\langle U_1,\ldots,U_n\rangle=\langle u_1,\ldots,u_n\rangle$, effect $\varphi$, candidate cause $\langle X_1,\ldots,X_\ell\rangle = \langle x_1,\ldots,x_\ell\rangle$, evaluation $\langle V_1,\ldots,V_m\rangle=\langle v_1,\ldots,v_m\rangle$}	
		\Function{FulfillsAC3}{$M, \vec{U}=\vec{u}, \varphi, \vec{X} = \vec{x}, \vec{V} = \vec{v}$}
		\If{$\ell > 1 \land (M, \vec{u}) \models \varphi$}\label{alg:line:fulfillsac3_sat:singleton_phi}
		\State \label{alg:line:fulfillsac3_sat:g}
		$G := \neg\varphi \wedge \bigwedge_{i=1\ldots n} f(U_i=u_i)$  $\wedge\bigwedge_{i=1\ldots m, \not\exists j\bullet X_j=V_i} \left(V_i \leftrightarrow F_{V_i} \lor f(V_i=v_i)\right)$ \WRP
		$	\wedge\bigwedge_{i=1\ldots\ell} X_i\lor\neg X_i$
		%	\If{$(\{\langle \vec U=\vec u, \vec V=\vec{v}'\rangle \:|\: \langle \vec U=\vec u, \vec V=\vec{v}'\rangle \in \textit{SAT}(\textit{CNF}(G))\})\not=\emptyset$}
		%	\label{alg:line:fulfillsac3_sat:sat_solve}
		%		\If{$(A_G:=\{\langle \vec U=\vec u, \vec V=\vec{v}'\rangle \:|\: \langle \vec U=\vec u, \vec V=\vec{v}'\rangle$ \WRP $\in \textit{SAT}(\textit{CNF}(G))\})\not=\emptyset$} 
		\label{alg:line:fulfillsac3_sat:satisfying_assignments}
		\ForAll{$\langle \vec U=\vec u, \vec V=\vec{v}'\rangle \in \textit{SAT}(\textit{CNF}(G))$}
		\If{$|\bigl\{j\in\{1,..,\ell\}|\exists i\bullet V_i=X_j\land v_i'\not=v_i$ \WRP$\land v_i'\not= [\overrightarrow{V} \mapsto \vec{v}'] F_{X_j}\bigr\}| < \ell$} \label{alg:line:fulfillsac3_sat:new_cause}
		\Return{\textit{false}}
		\EndIf
		\EndFor
		%	\EndIf
		\EndIf
		\State \Return{\textit{true}}
		\EndFunction
	\end{algorithmic}
\end{algorithm}

%\begin{reptheorem}{ac3}\label{theo:ac3}
%	Algorithm.\ref{algorithm:fulfillsac3_sat} returns 0 iff $\vec{X}$ is a non-minimal cause.
%\end{reptheorem}

\paragraph{\textbf{Combining AC2 and AC3}}\label{sec:ac23}
While developing Algorithm \ref{algorithm:fulfillsac2_sat} and Algorithm \ref{algorithm:fulfillsac3_sat}, we discovered that combining both is an option for optimizing our approach. In particular, we can exploit the relationship between the satisfying assignment(s) for the formulas $F$ and $G$, i.e., $\vec{a}_F \in A_G$. This holds, as we allow the variables $\vec{X}$ of a cause to be both $1$ or $0$ in $G$ so that we can show that the satisfying assignment, $\vec{a}_F$ for $F$ in Algorithm \ref{algorithm:fulfillsac2_sat} is an element of the satisfying assignments $A_G$, for $G$. Then, instead of computing both $F$ and $G$, we could just compute $G$, then filter those satisfying assignments that $F$ would have yielded and use them for checking AC2 while we use all satisfying assignments of $G$ to check AC3.

% !TeX root = main.tex
\subsection{Example} \label{subsec:ex}

Recall the example from Section.\ref{sec:preliminaries}.  Since the context $\vec{u}$ sets $ST = 1$ and $BT = 1$, the original evaluation of the model is shown in the first row of Table.\ref{tab:SB-values}. We want to find out whether $ST = 1$ is a cause of $BS = 1$. Algorithm.\ref{algorithm:fulfillsac2_sat} generates the following $F$, that is satisfiable for one assignment (Table.\ref{tab:SB-values} second row): $BS=0$,	$SH=0$,	$BH=0$,	$ST=0$, $BT=1$. All the variables, except  $BH$ and $BT$, change their evaluation. Thus, we conclude that $ST = 1$ fulfills AC2 with $\vec{W} = \{BH, BT\}$. Notice that even though this $\vec{W}$ is not minimal, it is still valid. That said, we still can calculate a minimal $\vec{W}$.% with more processing as described.

	\begin{multline*}
	F = \overbrace{\neg BS}^{\neg\varphi} \;\land\; \overbrace{ST_{exo} \land BT_{exo}}^{\vec{u}}
	\;\land\; (\overbrace{(BS \leftrightarrow SH \lor BH)}^{\text{equation of BS}} \lor \overbrace{BS}^{\mathclap{\text{orig. BS}}}) 
	\;\land\;(\underbrace{(SH \leftrightarrow ST)}^{\text{equation of SH}} \lor \underbrace{SH}^{\mathclap{\text{orig. SH}}}) \\ 
	\land\; (\underbrace{(BH \leftrightarrow BT \land \neg SH)}_{\text{equation of BH}} \lor \underbrace{\neg BH}_{\mathclap{\text{orig. BH}}}) 
	\;\land \underbrace{\neg ST}_{\text{equation of ST}}
	\land (\underbrace{(BT \leftrightarrow BT_{exo})}_{\text{equation of BT}} \lor \underbrace{BT}_{\mathclap{\text{orig. BT}}})
	\end{multline*}

\begin{table}[t]
\caption{$F$, $G$ assignments}\label{tab:SB-values}
	\centering
		\begin{tabular}{ |l | c | c | c | c | c| }
	\hline
	&$BS$	&$SH$	&$BH$	&$ST$	&$BT$ \\ \hline
	$M$ 								&1			&1 		&0 		&1 		&1 		\\
	$F$  	&0 		&0 		&0 		&0 		&1	\\
	$G$	   $\vec{a}_1$ 	&0 		&0 		&0 		&0 		&0	\\
	$G$	  $\vec{a}_2$	&0		&0 		&0 		&0 		&1 	\\\hline
\end{tabular}
\end{table}

To illustrate checking AC3, we ask a different question: are $ST = 1$ \textit{and} $BT = 1$ a cause of $BS = 1$?  Note that AC2 is fulfilled with $W = \emptyset$, for this cause. Obviously, if both do not throw, the bottle does not shatter. Using Algorithm \ref{algorithm:fulfillsac3_sat}, we obtain the following $G$ formula.

\begin{multline*}
	G = \overbrace{\neg BS}^{\neg\varphi} \;\land\; \overbrace{ST_{exo} \land BT_{exo}}^{\vec{u}}
	\;\land\; (\overbrace{(BS \leftrightarrow SH \lor BH)}^{\text{equation of $BS$}} \lor \overbrace{BS}^{\mathclap{\text{orig. $BS$}}}) 
	\;\land\; (\overbrace{(SH \leftrightarrow ST)}^{\text{equation of $SH$}} \lor \overbrace{SH}^{\mathclap{\text{orig. $SH$}}}) \\ \;\land 
	(\underbrace{(BH \leftrightarrow BT \land \neg SH)}_{\text{equation of $BH$}} \lor \underbrace{\neg BH}_{\mathclap{\text{orig. $BH$}}}) 
	\;\land\;  (\underbrace{ST}_{\text{orig. $ST$}} \hspace{2mm}\;\lor\;\hspace{2mm} \underbrace{\neg ST}_{\mathclap{\text{negated orig. $ST$}}})
	\;\land\; (\underbrace{BT}_{\text{orig. $BT$}} \hspace{2mm}\;\lor\;\hspace{2mm} \underbrace{\neg BT}_{\mathclap{\text{negated orig. $BT$}}})
	\end{multline*}

 As Table.\ref{tab:SB-values} shows,  $G$ is satisfiable with \textit{two} assignments $\vec{a}_1$ and $\vec{a}_2$. For $\vec{a}_1$, we can see that both $ST$ and $BT$ have values different from their original evaluation, and that both do not evaluate according to their equations. Thus, we cannot show that AC3 is violated, yet. For $\vec{a}_2$,  $BT = 1$, so it is equal to the evaluation of its equation. Consequently, $BT$ is not a required part of $\vec{X}$, because $\neg\varphi = \neg BS$ still holds although we did not  set $BT = 0$.  So, AC3 is not fulfilled because AC1 and AC2 hold for a subset of the cause.

\section{Evaluation}\label{sec:evaluation}
In this section, we provide details on the implementation of our algorithms, and  evaluate their efficiency.

\subsection{Technical implementation}
Our implementation is a Java library. As such, it can easily be  integrated into other systems. It supports both the creation of binary causal models as well as solving causality problems. For the modeling part and the implementation of our SAT-based approach, we take advantage of the library \textit{LogicNG} \footnote{\url{https://github.com/logic-ng/LogicNG}} . It provides methods for creating and modifying boolean formulas and allows to analyze those using different SAT solvers. We use the implementation of \textit{MiniSAT} solver  \cite{een2003extensible} within \textit{LogicNG}. For the sake of this evaluation, we will compare the \textit{execution time} and \textit{memory allocation} for the following \textit{four} strategies: \texttt{BRUTE\_FORCE} -a standard brute-force implementation of HP, \texttt{SAT} -SAT-based approach (Algorithm\ref{algorithm:fulfillsac2_sat}, \ref{algorithm:fulfillsac3_sat}), \texttt{SAT\_MINIMAL} -the minimal $\vec{W}$ extension, and \texttt{SAT\_COMBINED} -optimization of the SAT-based approach by combining AC2 and AC3. All our measurements were performed on Ubuntu 16.04 LTS machine equipped with an Intel\textregistered{} Core\texttrademark{} i7-3740QM CPU and 8 GB RAM. For each benchmark, we specified 100 warmup and 100 measurement iterations.

\subsection{Methodology and Evaluated Models}
 In summary, we experimented with 12 different models. On the one hand, we took the binary models from \cite{halpern2015modification}. There were 5 of them in total, namely, \textit{Rock-Throwing}, \textit{Forest Fire}, \textit{Prisoners}, \textit{Assassin}, and \textit{Railroad}.  Since these examples are rather small ($\le$ 5 variables) and easy to understand, they mainly serve for sanity checks of our approach. On the other hand, we used examples that do not stem from the literature on causality. One is a security example obtained from an industrial partner. It describes the causal factors that lead to stealing a security master key by an insider. We refer to it as \textit{SMK}. We used two variants of that example, one with 3 suspects, and the other with 8 suspects. 
 We also used one example from the safety domain that describes a leakage in a sub-sea production system; we refer to it as \textit{LSP}. Last, we artificially generated models of binary trees with different heights, denoted as \textit{$BT_{depth}$}, and combined them with other  random models (non-tree graphs), denoted as \textit{ABT}. For a thorough description of each model, please refer to this report \cite{models}. Table \ref{tab:evaluated_causal_models} shows the list of the bigger models, along with the number of endogenous variables. In total, we analyzed 278 scenarios, however, for space limits, we focus on a subset of the scenarios. 
 
\begin{table}[t]
		\caption{Evaluated causal models}\label{tab:evaluated_causal_models}
	\centering
	\begin{tabular}{| l |  c |}
		\hline
		\textbf{Causal Model} &\textbf{Endogenous Vars.} \\ \hline
		Abstract Models: $AM_1$ , $AM_2$ &8 , 3\\ 
		Steal Master Key: 3 suspects ($SMK_3$), 8 suspects ($SMK_8$)  &36, 91 \\ 
		Leakage in Sub-sea Production System (LSP) \cite{cheliyan2018fuzzy} &41 \\ 
		Binary Trees of different heights (BT) & 15 - 4095 \\ 
		Abstract Model 1 Combined with Binary Tree (ABT) & 4103 \\ \hline
	\end{tabular}

\end{table}

\setcounter{rowcounter}{0}
\DTLloaddbtex{\rockthrowingselected}{data/brute_vs_sat.dbtex}
{\begin{table}[h]
		\centering
		\caption{Discussed scenarios as part of the analysis}\label{tab:evaluation:brute_sat}
\verytiny
\setlength\tabcolsep{4pt}
\resizebox{\columnwidth}{!}{%
%we hide the context column! replace H by | >{\centering\arraybackslash}p{2.5cm} to make it visible
\begin{tabular}{>{\centering\arraybackslash}p{0.7cm} | l H | >{\centering\arraybackslash}p{4.0cm} | c || >{\centering\arraybackslash}p{0.15cm} | >{\centering\arraybackslash}p{0.15cm} | >{\centering\arraybackslash}p{0.15cm} | c | >{\centering\arraybackslash}p{1.0cm} | >{\centering\arraybackslash}p{1.0cm} | >{\centering\arraybackslash}p{1.0cm} | >{\centering\arraybackslash}p{1.0cm}|}
&ID &$\vec{u}$ &$\vec{X} = \vec{x}$ &$\varphi$ &\parbox[t]{1.5mm}{\rotatebox[origin=c]{90}{AC1}} &\parbox[t]{1.5mm}{\rotatebox[origin=c]{90}{AC2}} &\parbox[t]{1.5mm}{\rotatebox[origin=c]{90}{AC3}} &\parbox[t]{3mm}{\rotatebox[origin=c]{90}{$\abs{\text{Minimal } \vec{W}}$}} 
&\parbox[t]{3mm}{\rotatebox[origin=c]{90}{BRUTE}}
&\parbox[t]{3mm}{\rotatebox[origin=c]{90}{SAT}}
&\parbox[t]{3mm}{\rotatebox[origin=c]{90}{SAT$^{\text{MINIMAL}}$}}
&\parbox[t]{3mm}{\rotatebox[origin=c]{90}{SAT$_{\text{COMBINED}}$}}\\ \hline \hline
\parbox[t]{5mm}{\multirow{14}{*}{\rotatebox[origin=c]{90}{\parbox{2cm}{\centering $SMK_3$}}}}
&3 &\multirow{6}{*}{$\forall U \in \mathcal{U}: U = 1$}
&$FS_{U_1} = 1 \land FN_{U_1} = 1\land\: A_{U_1} = 1$ &\multirow{8}{*}{\parbox[t]{1.5mm}{\rotatebox[origin=c]{90}{$SMK = 1$}}} &\cmark &\cmark &\cmark &4 \csvEight{73} \\ \cline{2-2}\cline{4-4}\cline{6-12}
&22 &\multirowcell{8}{$\forall U_{U_1} \in \mathcal{U}: U_{U_1} = 0$\\$\forall U_{U_2} \in \mathcal{U}: U_{U_2} = 0$\\$\forall U_{U_3} \in \mathcal{U}: U_{U_3} = 1$}
&$FS_{U_3} = 1$ &&\cmark &\xmark &\cmark &-- \csvEight{89} \\ \cline{2-2}\cline{4-4}\cline{6-12}
&24 &&$FS_{U_3} = 1 \land FN_{U_3} = 1\land\: A_{U_3} = 1$ &&\cmark &\cmark &\cmark &0 \csvEight{97} \\ \cline{2-2}\cline{4-4}\cline{6-12}
&26 &&$FS_{U_3} = 1 \land FN_{U_3} = 1$\newline$\land\: A_{U_3} = 1 \land AD_{U_3} = 1$ &&\cmark &\cmark &\xmark &0 \csvEight{105} \\ \cline{2-2}\cline{4-12}
&29 &&$A_{U_3} = 1 \land AD_{U_3} = 1$\newline\newline\newline &\multirow{4}{*}{\parbox[t]{1.5mm}{\rotatebox[origin=c]{90}{$SD = 1$}}}&\cmark &\cmark &\xmark &0 \csvEight{113} \\ \hline
\parbox[t]{5mm}{\multirow{6}{*}{\rotatebox[origin=c]{90}{\parbox{1.2cm}{\centering LSP}}}}
&3&\verytiny\multirowcell{2}{$X^{exo}_1, X^{exo}_2 = 1$\\ $\forall i \not\in \{1,2\}: X^{exo}_i = 0$}
&$X_1 = 1 \land X_2 = 1$ &\multirow{6}{*}{\parbox[t]{1.5mm}{\rotatebox[origin=c]{90}{$X_{41} = 1$}}} &\cmark &\cmark &\xmark &0 \csvEight{137} \\ \cline{2-2}\cline{4-4}\cline{6-12}
&56&\verytiny\multirowcell{4}{$X^{exo}_{1}, X^{exo}_{2},$\\$X^{exo}_{3}, X^{exo}_{11} = 1$\\ $\forall i \not\in \{1,2,3,11\}:$\\$X^{exo}_i = 0$}
&$X_{1} = 1 \land X_{2} = 1$ &&\cmark &\xmark &\cmark &-- \csvEight{145} \\ \cline{2-2}\cline{4-4}\cline{6-12}
&57&&$X_{1} = 1 \land X_{3} = 1$ &&\cmark &\cmark &\cmark &0 \csvEight{153} \\ \hline
\parbox[t]{8mm}{\multirow{6}{*}{\rotatebox[origin=c]{90}{\parbox{1cm}{\centering\verytiny BT \\ height = 12}}}}
&34 &\multirow{8}{*}{$\forall U \in \mathcal{U}: U = 1$}
&$n_{4093} = 1 \land n_{4094} = 1$\newline\newline &\multirow{6}{*}{\parbox[t]{1.5mm}{\rotatebox[origin=c]{90}{$n_{\text{root}} = 1$}}} &\cmark &\xmark &\cmark &-- \csvEight{209} \\ \cline{2-2}\cline{4-4}\cline{6-12}
&35&&$n_{4091} = 1 \land n_{4092} = 1$\newline$\land\: n_{4093} = 1 \land n_{4094} = 1$\newline &&\cmark &\xmark &\cmark &-- \csvEight{217} \\ \hline

\parbox[t]{8mm}{\multirow{6}{*}{\rotatebox[origin=c]{90}{\parbox{1.2cm}{\centering\verytiny ABT}}}}
&1&\verytiny\multirowcell{2}{\verytiny$B_{exo}, n^{exo}_{4094} = 1$\\ $\forall i \not\in \{4094\}: n^{exo}_i = 0$}
&$n_{4094} = 1$ &\multirow{6}{*}{\parbox[t]{1.5mm}{\rotatebox[origin=c]{90}{$I = 1$}}} &\cmark &\cmark &\cmark &4 \csvEight{233} \\ \cline{2-2}\cline{4-4}\cline{6-12}
&4&\multirowcell{4}{\verytiny$B_{exo}, n^{exo}_{4093}, n^{exo}_{4094} = 1$\\\verytiny $\forall i \not\in \{4093, 4094\}:$\\\verytiny $n^{exo}_i = 0$}
&$n_{4093} = 1 \land n_{4094} = 1$ &&\cmark &\cmark &\cmark &4 \csvEight{241} \\ \cline{2-2}\cline{4-4}\cline{6-12}
&5&&$n_{4092} = 0 \land n_{4093} = 1\land n_{4094} = 1$ &&\cmark &\cmark &\xmark &5 \csvEight{249} \\ \hline
\end{tabular}}
\end{table}}

\subsection{Discussion and Results}
In this section, we discuss some representative cases from our experiments. Table \ref{tab:evaluation:brute_sat} shows the details of these cases. The first two columns show the scenario identifier, namely, the \textit{name} of the model and the \textit{ID} of the scenario that  differs in the details of the causal query, i.e., $\vec{X}$ and $\varphi$. The latter are shown in the third and fourth columns. Then, the results of the three conditions are displayed in columns AC1-AC3. The size of the minimal W set is displayed in the next column. Finally, for each algorithm, the execution time and memory allocation are shown. We write N/A in cases where the computation was not completed in 5 minutes or consumed too much memory. As a general remark, it does not matter which algorithm is applied if AC2 holds for an empty $\vec{W}$ and AC3 holds automatically, i.e., $\vec{X}$ is a singleton.

%- same for all cases 
%- write that BRUTE is not necessarily always worse
As expected, the Brute-Force approach (BF) works only for smaller models (\textless 5 variables), or in situations where only a few iterations are required. Such as scenarios \textit{LSP-3, and SMK-29}. Specifically, we see these situations when AC2 holds with a small or empty $\vec{W}$, and AC3 does not hold. That is, the number of iterations BF performs is small because the sets, $\vec{W}_i$ are ordered by size. %The second factor is the case that . This requires a few iterations within BF and just takes less time than computing and analyzing $G$ as part of the SAT algorithm. 
Such examples did not exhibit the major problem of BF, i.e., the generation of all possible sets, $\vec{W}_i$ whose number increases exponentially, and the iterations BF might, therefore, perform to check minimality in AC3.

%SAT better when ac2 doesn't hold
For larger models ( \textgreater 30 variables), BF did not return an answer in 5 minutes, especially when AC2 does not hold. This is seen by the several N/A entries in Table \ref{tab:evaluation:brute_sat}. For example, in the SMK, the set of all possible $\vec{W}_i$ has a size of up to $2^{35}$. In the worst case, this number of iterations is required for finding out that AC2 does not hold. It is possible that this number of iterations multiplied by the number of subsets of the cause needs to be executed again to check AC3. This causes BF to be extremely slow and to consume a lot of memory. The SAT by contrast, always stays below 1.5 ms and allocates less than 1.5 MB during the execution for all scenarios of the SMK.  Even if larger models were considered, SAT handles them efficiently, e.g., \textit{BT 34 - 35}, where the underlying causal model contains 4095 variables, executed in $\le 7s$. However, the latter scenarios are special because AC2 does not hold. In \textit{ABT 1, 4 and 5}, we can see that even if AC2 \textit{does} hold and $\vec{W}$ is \textit{not} empty, SAT takes only $8s$.

%minmality
 While obtaining a minimal $\vec{W}$ using our approach showed a rather small impact relative to the SAT approach in most of the scenarios, it showed a significant increase in some scenarios. This impact was highly dependent on the nature and semantics of the underlying causal model. That is, we can only observe a major impact if the number of satisfying assignments or the size of a non-minimal $\vec{W}$ is large as this will significantly extend the analysis. For instance, in \textit{SMK-3}, the execution time increased by about $22\%$. Nonetheless, there are scenarios in which we observed a significant increase, such as the \textit{ABT-4} scenario. Here, the non-minimal $\vec{W}$ contains more than 4000 elements, leading to an increase of more than 200\% in the execution time required to determine a minimal $\vec{W}$. 

%ac23
Finally, combining the algorithms for AC2 and AC3 is only beneficial if AC2 and AC3 need to be explicitly analyzed (AC2 does not hold for an empty W, and the cause is not a singleton). We have many such scenarios in our examples. In evaluating them, we found out that there is a positive impact in using this optimization, but it is rather small on the average. Larger differences can be seen, for instance, in \textit{ABT-5} where the SAT-based approach executes for $7906 ms$ while the current optimization takes $3803ms$.

The main finding of our experiments is that actual causality can be computed efficiently with our SAT-based approach. Within binary models of \textit{4000} variables, we were able to obtain a correct answer for any query in less than $4$ seconds, using a memory of \textit{1 GB}. 
\section{Conclusions and Future Work}\label{sec:conc}
It is difficult to devise automated assistance for causality reasoning in modern socio-technical systems. Causality checking, according to the formal definitions, is computationally hard. Therefore, efficient approaches that scale to the complexity of such systems are required. In the course of this, we proposed an intelligent way to utilize SAT solvers to check actual causality in binary models in a large scale that we believe to be particularly relevant for accountability purposes. We empirically showed that it can efficiently compute actual causality in large binary models. Even with only 30 variables, determining causality in a brute force manner is incomputable, whereas our SAT-based approach returned a result for such cases in 1 ms. In addition, causal models consisting of more than 4000 endogenous variables were still handled within seconds using the proposed approach. %We have also shown one idea to enhance the quality of the answer, i.e., a minimal $\vec{W}$, and one idea to enhance the performance, i.e., combining the two algorithms. 

For future work, we will consider other logic programming paradigms such as integer linear programming and answer set programming to develop our approach from checking to possibly inferring causality. Moreover, a thorough characterization of causal model classes that affect the efficiency of the proposed approach is a useful follow-up to this work. We plan to extend our benchmark with models of different patterns, and  different cardinalities of causes.

\newpage
 \bibliographystyle{splncs04}
 \bibliography{bibliography}

\begin{thebibliography}{10}
\providecommand{\url}[1]{\texttt{#1}}
\providecommand{\urlprefix}{URL }
\providecommand{\doi}[1]{https://doi.org/#1}

\bibitem{aleksandrowicz2014computational}
Aleksandrowicz, G., Chockler, H., Halpern, J.Y., Ivrii, A.: The computational
  complexity of structure-based causality. In: Proceedings of the Twenty-Eighth
  {AAAI} Conference on Artificial Intelligence, July 27 -31, 2014, Qu{\'{e}}bec
  City, Qu{\'{e}}bec, Canada. pp. 974--980 (2014),
  \url{http://www.aaai.org/ocs/index.php/AAAI/AAAI14/paper/view/8328}

\bibitem{beer2015symbolic}
Beer, A., Heidinger, S., K{\"{u}}hne, U., Leitner{-}Fischer, F., Leue, S.:
  Symbolic causality checking using bounded model checking. In: Model Checking
  Software - 22nd International Symposium, {SPIN} 2015, Stellenbosch, South
  Africa, August 24-26, 2015, Proceedings. pp. 203--221 (2015)

\bibitem{beer2012explaining}
Beer, I., Ben{-}David, S., Chockler, H., Orni, A., Trefler, R.J.: Explaining
  counterexamples using causality. Formal Methods in System Design
  \textbf{40}(1),  20--40 (2012). \doi{10.1007/s10703-011-0132-2},
  \url{https://doi.org/10.1007/s10703-011-0132-2}

\bibitem{bertossi2018characterizing}
Bertossi, L.: Characterizing and computing causes for query answers in
  databases from database repairs and repair programs. In: International
  Symposium on Foundations of Information and Knowledge Systems. pp. 55--76.
  Springer (2018)

\bibitem{cheliyan2018fuzzy}
Cheliyan, A.S., Bhattacharyya, S.K.: Fuzzy fault tree analysis of oil and gas
  leakage in subsea production systems. Journal of Ocean Engineering and
  Science  \textbf{3}(1),  38 -- 48 (2018).
  \doi{https://doi.org/10.1016/j.joes.2017.11.005},
  \url{http://www.sciencedirect.com/science/article/pii/S2468013317300591}

\bibitem{chockler2004responsibility}
Chockler, H., Halpern, J.Y.: Responsibility and blame: {A} structural-model
  approach. J. Artif. Intell. Res.  \textbf{22},  93--115 (2004).
  \doi{10.1613/jair.1391}, \url{https://doi.org/10.1613/jair.1391}

\bibitem{chockler2008causes}
Chockler, H., Halpern, J.Y., Kupferman, O.: What causes a system to satisfy a
  specification? ACM Transactions on Computational Logic (TOCL)  \textbf{9}(3),
  ~20 (2008)

\bibitem{een2003extensible}
E{\'{e}}n, N., S{\"{o}}rensson, N.: An extensible sat-solver. In: Theory and
  Applications of Satisfiability Testing, 6th International Conference, {SAT}
  2003. Santa Margherita Ligure, Italy, May 5-8, 2003 Selected Revised Papers.
  pp. 502--518 (2003)

\bibitem{feigenbaum2011towards}
Feigenbaum, J., Jaggard, A.D., Wright, R.N.: Towards a formal model of
  accountability. In: 2011 New Security Paradigms Workshop, {NSPW} '11, Marin
  County, CA, USA, September 12-15, 2011. pp. 45--56 (2011).
  \doi{10.1145/2073276.2073282},
  \url{http://doi.acm.org/10.1145/2073276.2073282}

\bibitem{granger:80}
Granger, C.: Investigating causal relations by econometric models and
  cross-spectral methods. Econometrica  \textbf{37}(3) (1969)

\bibitem{halpern2015modification}
Halpern, J.Y.: A modification of the halpern-pearl definition of causality. In:
  Proceedings of the Twenty-Fourth International Joint Conference on Artificial
  Intelligence, {IJCAI} 2015, Buenos Aires, Argentina, July 25-31, 2015. pp.
  3022--3033 (2015), \url{http://ijcai.org/Abstract/15/427}

\bibitem{halpern2016actual}
Halpern, J.Y.: Actual causality. The MIT Press, Cambridge, Massachussetts
  (2016)

\bibitem{halpern2018towards}
Halpern, J.Y., Kleiman-Weiner, M.: Towards formal definitions of
  blameworthiness, intention, and moral responsibility. In: Proceedings of the
  Thirty-Second AAAI Conference on Artificial Intelligence (AAAI-18) (2018)

\bibitem{hopkins2002strategies}
Hopkins, M.: Strategies for determining causes of events. In: AAAI/IAAI. pp.
  546--552 (2002)

\bibitem{hume1748An}
Hume, D.: An Enquiry Concerning Human Understanding (1748)

\bibitem{models}
Ibrahim, A.: Efficient checking of actual causality via sat solving -
  benchmarked models,
  \url{{https://github.com/amjadKhalifah/HP2SAT1.0/blob/master/doc/models.pdf}}

\bibitem{kacianka2016towards}
Kacianka, S., Kelbert, F., Pretschner, A.: Towards a unified model of
  accountability infrastructures. In: Proceedings First Workshop on Causal
  Reasoning for Embedded and safety-critical Systems Technologies, CREST\@
  ETAPS 2016, Eindhoven, The Netherlands, 8th April 2016. pp. 40--54 (2016).
  \doi{10.4204/EPTCS.224.5}, \url{https://doi.org/10.4204/EPTCS.224.5}

\bibitem{kacianka2018understanding}
Kacianka, S., Pretschner, A.: Understanding and formalizing accountability for
  cyber-physical systems. In: 2018 IEEE International Conference on Systems,
  Man, and Cybernetics (SMC). pp. 3165--3170. IEEE (2018)

\bibitem{kleinberg2011review}
Kleinberg, S., Hripcsak, G.: A review of causal inference for biomedical
  informatics. Journal of Biomedical Informatics  \textbf{44}(6),  1102--1112
  (2011). \doi{10.1016/j.jbi.2011.07.001},
  \url{https://doi.org/10.1016/j.jbi.2011.07.001}

\bibitem{kuennemann2018automated}
K{\"{u}}nnemann, R., Esiyok, I., Backes, M.: Automated verification of
  accountability in security protocols. CoRR  \textbf{abs/1805.10891} (2018),
  \url{http://arxiv.org/abs/1805.10891}

\bibitem{leitner-fischer2015causality}
Leitner{-}Fischer, F.: Causality Checking of Safety-Critical Software and
  Systems. Ph.D. thesis, University of Konstanz, Germany (2015),
  \url{http://kops.uni-konstanz.de/handle/123456789/30778}

\bibitem{leitner-fischer2013causality}
Leitner{-}Fischer, F., Leue, S.: Causality checking for complex system models.
  In: Verification, Model Checking, and Abstract Interpretation, 14th
  International Conference, {VMCAI} 2013, Rome, Italy, January 20-22, 2013.
  Proceedings (2013)

\bibitem{lewis1973causation}
Lewis, D.: Causation. Journal of Philosophy  \textbf{70}(17),  556--567 (1973).
  \doi{10.2307/2025310}

\bibitem{meliou2010causality}
Meliou, A., Gatterbauer, W., Halpern, J.Y., Koch, C., Moore, K.F., Suciu, D.:
  Causality in databases. {IEEE} Data Eng. Bull.  \textbf{33}(3),  59--67
  (2010), \url{http://sites.computer.org/debull/A10sept/suciu.pdf}

\bibitem{meliou2010complexity}
Meliou, A., Gatterbauer, W., Moore, K.F., Suciu, D.: The complexity of
  causality and responsibility for query answers and non-answers. {PVLDB}
  \textbf{4}(1),  34--45 (2010),
  \url{http://www.vldb.org/pvldb/vol4/p34-meliou.pdf}

\bibitem{DBLP:journals/computer/Meyer92}
Meyer, B.: Applying "design by contract". {IEEE} Computer  \textbf{25}(10),
  40--51 (1992). \doi{10.1109/2.161279}

\bibitem{moore2009causation}
Moore, M.S.: Causation and responsibility : an essay in law, morals, and
  metaphysics. Oxford Univ. Press, Oxford (2009)

\bibitem{newsham2014impact}
Newsham, Z., Ganesh, V., Fischmeister, S., Audemard, G., Simon, L.: Impact of
  community structure on sat solver performance. In: International Conference
  on Theory and Applications of Satisfiability Testing. pp. 252--268. Springer
  (2014)

\bibitem{pearl2018theoretical}
Pearl, J.: Theoretical impediments to machine learning with seven sparks from
  the causal revolution. arXiv preprint arXiv:1801.04016  (2018)

\bibitem{DBLP:journals/ai/Reiter87}
Reiter, R.: A theory of diagnosis from first principles. Artif. Intell.
  \textbf{32}(1),  57--95 (1987)

\bibitem{salimi2014causes}
Salimi, B., Bertossi, L.: From causes for database queries to repairs and
  model-based diagnosis and back. arXiv preprint arXiv:1412.4311  (2014)

\bibitem{stachowiak}
Stachowiak, H.: Allgemeine Modelltheorie. Springer (1973), available online:
  \url{https://archive.org/details/Stachowiak1973AllgemeineModelltheorie}

\end{thebibliography}
 \newpage
 \appendix
 \section{Appendix: Lemma Proof} \label{sec:app1}
\begin{lemma}\label{lemma:negation2}
	In a binary model, if $\vec{X}=\vec{x}$ is a cause of $\varphi$, according to HP \cite{halpern2015modification} definition, then every $\vec{x}'$ in the definition of AC2 always satisfies $\forall i\bullet x_i'=\neg x_i$.
\end{lemma}

\begin{proof}
	We use the following notation: $\ora{X}_{(n)}$ stands for a vector of length $n$, $X_1,\ldots,X_n$; and $\ora{X}_{(n)}=\ora{x}_{(n)}$ stands for $X_1=x_1,\ldots,X_n=x_n$. Let $\ora{X}_{(n)}=\ora{x}_{(n)}$ be a cause for $\varphi$ in some model $M$.
	
	\begin{enumerate}
		\item AC1 yields 
		\begin{equation}\label{eq:origAC1}
			(M,\ora{u})\models (\ora{X}_{(n)}=\ora{x}_{(n)}) \wedge (M,\ora{u})\models\varphi.
		\end{equation}
		\item Assume that the lemma does not hold. Then there is some index $k$ such that $x_k'=x_k$ and AC2 holds. Because we are free to choose the ordering of the variables, let us set $k=n$ wlog.
		We may then rewrite AC2 as follows: 
		\begin{multline}\label{eq:AC2forLemma}
			\exists \ora{W}, \ora{w}, \ora{x}_{(n)}'\bullet  (M,\ora{u})\models(\ora{W}=\ora{w})
			\implies (M,\ora{u})\models \\\bigl[\ora{X}_{(n-1)}\leftarrow \ora{x}_{(n-1)}',X_n\leftarrow x_n,\ora{W}\leftarrow\ora{w}\bigr]\neg\varphi.
		\end{multline}
		
		\item We will show that equations \ref{eq:origAC1} and \ref{eq:AC2forLemma} give rise to a smaller cause, namely $\ora{X}_{(n-1)}= \ora{x}_{(n-1)}$, contradicting the minimality requirement AC3. We need to show that the smaller cause $\ora{X}_{(n-1)}= \ora{x}_{(n-1)}$ satisfy AC1 and AC2, as stated by equations \ref{eq:newAC1} and \ref{eq:AC2Cons} below.
		% if we let $X_n=x_n$. 
		This violates the minimality requirement of AC3 for $\ora{X}_{(n)}=\ora{x}_{(n)}$.
		
		\begin{equation}
			\label{eq:newAC1}
			(M,\ora{u})\models (\ora{X}_{(n-1)}=\ora{x}_{(n-1)}) \wedge (M,\ora{u})\models\varphi
		\end{equation}
		
		states AC1 for a candidate ``smaller'' cause $\ora{X}_{(n-1)}$. Similarly,
		
		\begin{multline}	\label{eq:AC2Cons}
			\exists \ora{W}^*,\ora{w}^*,\ora{x}_{(n-1)}'^*\bullet(M,\ora{u})\models(\ora{W}^*=\ora{w}^*)\\
			\implies (M,\ora{u})\models\bigl[\ora{X}_{(n-1)}\leftarrow\ora{x}_{(n-1)}'^*,\ora{W}^*\leftarrow\ora{w}^*\bigr]\neg\varphi
		\end{multline}
		formulates AC2 for this candidate smaller cause $\ora{X}_{(n-1)}$.
		\item Let $\Psi$ denote the structural equations that define $M$. Let $\Psi'$ be $\Psi$ without the equations that define the variables $\ora{X}_{(n)}$ and $\ora{W}$; and let $\Psi''$ be $\Psi$ without the equations that define the variables $\ora{X}_{(n-1)}$ and $\ora{W}$. Clearly, $\Psi''\implies\Psi'$.
		
		We can turn equation \ref{eq:origAC1} into a propositional formula, namely
		\begin{equation}\label{eq:1}
			E_1:=\bigl(\Psi\wedge \ora{X}_{(n-1)}=\ora{x}_{(n-1)}\wedge X_n=x_n\bigr) \wedge \varphi.
		\end{equation}
		
		Similarly, equation \ref{eq:newAC1} is reformulated as
		\begin{equation}\label{eq:2}
			E_2:=\bigl(\Psi\wedge \ora{X}_{(n-1)}=\ora{x}_{(n-1)}\bigr) \wedge \varphi.
		\end{equation}
		
		Because equation~\ref{eq:AC2forLemma} holds, we fix some $\ora{W},\ora{w},\ora{x}_{(n)}'$ that make it true and rewrite this equation as
		\begin{equation}\label{eq:3}
			E_3:=\bigl(\Psi'\wedge\ora{X}_{(n-1)}=\ora{x}_{(n-1)}'\wedge X_n=x_n\wedge \ora{W}=\ora{w}\bigr)\implies \neg\varphi.
		\end{equation}

		Finally, in equation \ref{eq:AC2Cons}, we use exactly these values to also fix $\ora{W}^*=\ora{W}$, $\ora{w}^*=\ora{w}$, and $\ora{x}_{(n-1)}'^*=\ora{x}_{(n-1)}'$, and reformulate this equation as
		\begin{equation}\label{eq:4}
			E_4:=\bigl(\Psi''\wedge\ora{X}_{(n-1)}=\ora{x}_{(n-1)}'\wedge \ora{W}=\ora{w}\bigr)\implies \neg\varphi.
		\end{equation}
		
		It is then a matter of equivalence transformations to show that
		\begin{equation}
			(\Psi''\implies\Psi') \implies \bigl((E_1\wedge E2)\implies(E_3\wedge E_4)\bigr)
		\end{equation}
		is a tautology, which proves the lemma.
	\end{enumerate}
\end{proof}
\end{document}